\documentclass{article}

\usepackage[preprint, nonatbib]{neurips_2022}

\usepackage[utf8]{inputenc} 
\usepackage[T1]{fontenc}    
\usepackage{hyperref}       
\usepackage{url}            
\usepackage{booktabs}       
\usepackage{amsfonts}       
\usepackage{nicefrac}       
\usepackage{microtype}      
\usepackage{xcolor}         

\usepackage{multirow}
\usepackage{verbatim}
\usepackage{amssymb, amsmath, amsthm}
\usepackage[english]{babel}
\usepackage{lmodern, csquotes, eurosym}
\usepackage[style=authoryear,texencoding=utf8,backend=biber, minnames = 1, maxnames = 4]{biblatex}
\usepackage{graphicx,subcaption}
\usepackage{geometry}
\usepackage{pdflscape}
\usepackage{tikz}
\usetikzlibrary{arrows}
\usepackage{blkarray}
\usepackage{derivative}

\newtheorem{corollary}{Corollary}

\newtheorem{definition}{Definition}

\newtheorem{lemma}{Lemma}

\newtheorem{proposition}{Proposition}
\newtheorem{proposition2}{Proposition}

\DeclareMathOperator*{\argmin}{argmin}

\title{Calibrating for Class Weights\\by Modeling Machine Learning}

\author{%
  Andrew Caplin \\
  Department of Economics\\
  New York University \\
  \texttt{andrew.caplin@nyu.edu} \\
  \And
  Daniel Martin \\
  Kellogg School of Management \\
  Northwestern University \\
  \texttt{d-martin@kellogg.northwestern.edu} \\
  \AND
  Philip Marx \\
  Department of Economics \\
  Louisiana State University \\
  \texttt{philiplmarx@gmail.com} \\
}

\begin{document}

\maketitle

\begin{abstract}
A much studied issue is the extent to which the confidence scores provided by machine learning algorithms are calibrated to ground truth probabilities. Our starting point is that calibration is seemingly incompatible with class weighting, a technique often employed when one class is less common (class imbalance) or with the hope of achieving some external objective (cost-sensitive learning). We provide a model-based explanation for this incompatibility and use our anthropomorphic model to generate a simple method of recovering likelihoods from an algorithm that is miscalibrated due to class weighting. We validate this approach in the binary pneumonia detection task of \textcite{rajpurkar2017chexnet}.
\end{abstract}

\section{Introduction}

An important set of machine learning applications involve classification. In a classification task, the goal is to correctly identify a categorical label $y \in Y = \{0,\dots,n-1\}$ given an instance/observation $x$. For example, the label $y$ may be a medical condition, the instance $x$ a medical image, and the prediction a recommended clinical diagnosis. In addition to this prediction, machine learning classifiers typically also provide a vector of confidence scores $a = (a_0, ..., a_{n-1}) \in A \subseteq \mathbb{R}^n$ that can be used to assess ``confidence'' in their predictions.

A much studied issue is the extent to which these confidence scores are \emph{calibrated} to ground truth  probabilities. When scores are calibrated, they represent the likelihood of each label, for instance, the probability a medical image indicates a medical condition. Calibration has clear advantages for downstream purposes, as evidenced by the large literature devoted to its study (\cite{platt1999calibration}, \cite{zadroznyelkan2001}, \cite{zadroznyelkan2002}, \cite{guo2017calibration}, \cite{minderer2021calibration}). 
For example, if a score indicates high uncertainty about the true label, then alternative likelihoods, such as those produced by other programs or human agents, can be used instead (\cite{jiang2012medcal}, \cite{raghu2019algorithmic}, \cite{kompa2021medcal}). Also, calibrated confidence scores are valuable for human interpretability (\cite{cosmides1996humanstat}), which can improve decision making and trust in algorithmic predictions.

Our starting point is that calibration is seemingly incompatible with class weighting, a technique often employed when one class is less common (\cite{thai2010cost}) or with the hope of achieving some external objective (\cite{zadrozny2003cost}). We provide a model-based explanation for this incompatibility and use our model to generate a simple method of recovering likelihoods from an algorithm that is miscalibrated due to class weighting. 

In the anthropomorphic model that underpins our approach, we conceive of an algorithm as a decision-maker that receives informative signals, forms posterior beliefs, and reports confidence scores that minimize its loss function given its beliefs. We offer a simple test of this model based on the notion of \textit{loss-calibration}, which indicates how a loss function can incentivize miscalibration. If an algorithm passes this test, then a simple analytic correction can be used to generate \textit{loss-corrected confidence scores} that are calibrated. When applicable, such analytical corrections are arguably preferable to general rescaling methods such as \textcite{platt1999calibration} and \textcite{zadroznyelkan2001} because they are less prone to misspecification and do not require further estimation, improving efficiency. Also, our model-based approach provides an intuitive logic for extending analytical recalibration approaches to multiclass problems (Proposition \ref{thm:ap-multi}).

We validate our approach in the binary pneumonia detection task of \textcite{rajpurkar2017chexnet} using CheXNeXt, a state-of-the-art deep convolutional neural network for predicting thoracic diseases from chest X-ray images (\cite{rajpurkar2018chexnext}). We chose this task both because it represents a setting of societal importance and because it is a setting where class imbalance is endemic, as most chest X-rays do not indicate the presence of a thoracic disease such as pneumonia.

\section{Calibration and Class Weights}

\begin{figure}
    \centering
    \includegraphics[scale=0.3]{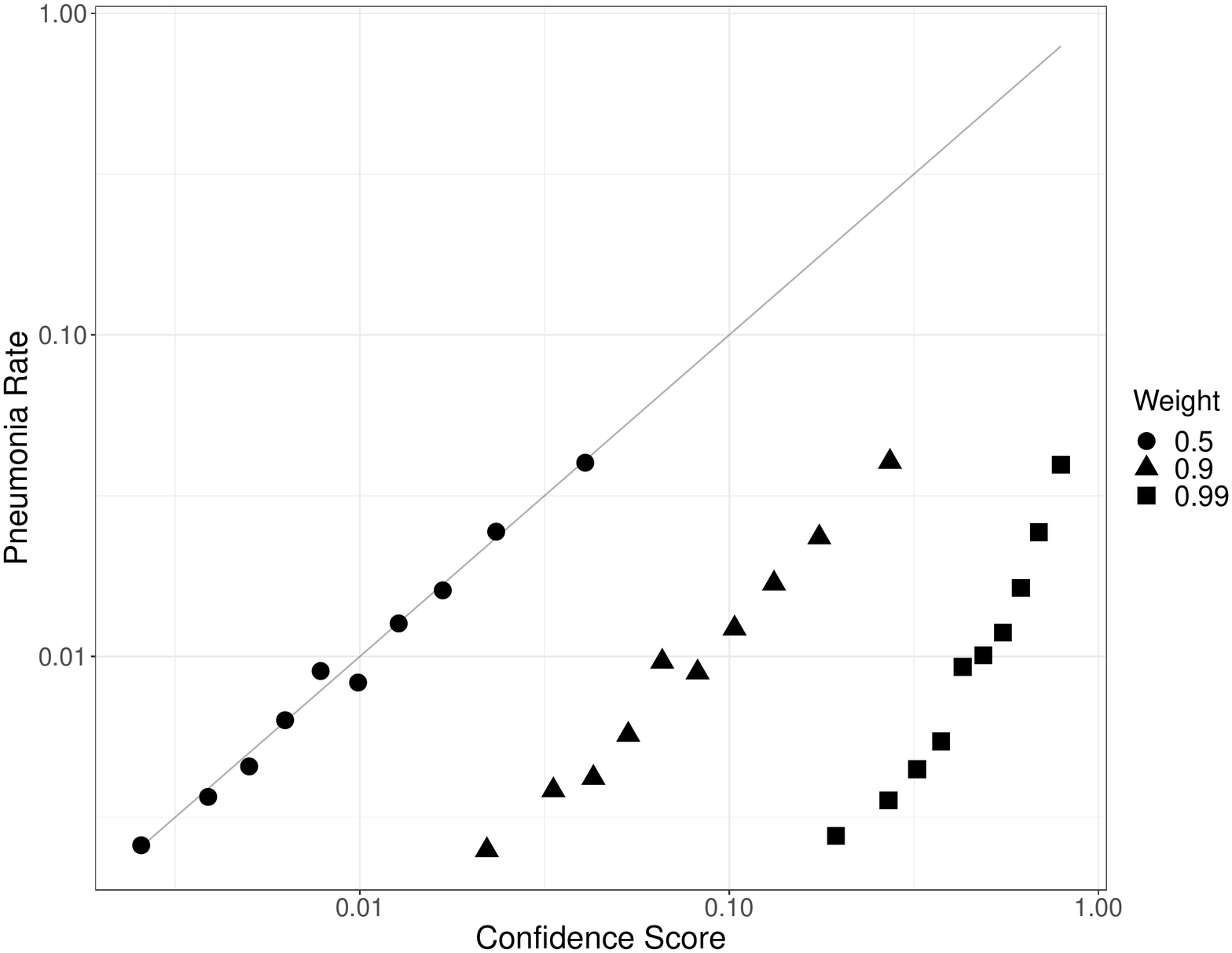}
    \caption{
    \label{fig:miscalibrated-data}
    Calibration curves for the pneumonia detection task of \textcite{rajpurkar2017chexnet} for ChestX-ray14 data (\cite{wang2017chestxray}) with varying class weights.
    }
\end{figure}

Formally, confidence scores are calibrated if, for each potential label and observed score, the score equals the probability of the label when that score is provided:

\begin{definition}[Calibration]
\label{def:calibration}
Confidence scores are \textbf{\textbf{calibrated}} if: 
\begin{equation}
\label{eq:calibration}
a_y = P(y | a_y)
\quad \text{for all $a_y$ such that $P(a_y)>0$.}
\end{equation}
\end{definition}

\noindent

We use the binary pneumonia detection task of \textcite{rajpurkar2017chexnet} to illustrate the incompatibility of calibration with class weighting. A standard log loss function for binary classification would specify losses from the positive confidence score $a_1$ conditional on a positive label (outcome $y=1$) as $- \log (a_1)$ and those on a negative label (outcome $y=0$) as $- \log (1-a_1)$. Class weighting allows for differential losses for different types of error, producing losses of $-\beta_1 y \log (a_1)-(1-\beta_1) (1-y) \log (1-a_1)$ where $\beta_1 \in (0,1)$ denotes the relative weight on the positive class. For a rare yet important event such as pneumonia, one upweights the relative importance of false negatives: the inverse class weight used in \textcite{rajpurkar2017chexnet} was approximately 0.99.

To experimentally evaluate the impact of class weighting and our procedures, we trained models using the ChestX-ray14 dataset, which consists of 112,120 frontal chest X-rays which were synthetically labeled with the presence of fourteen thoracic diseases (\cite{wang2017chestxray}).
Our main modifications are that we isolate the task of pneumonia detection as in the earlier implementation of \cite{rajpurkar2017chexnet}, and that we train the algorithm across various $\beta$-weighted cross-entropy loss functions. In addition, we employ ensemble (model-averaging) methods to isolate the substantive effects of what the machine learns from random noise inherent to the stochastic training procedure.
Using nested cross-validation methods, this yields an ensemble model prediction at each $\beta$ for each of the 112,120 X-ray images in the original data.
Further technical details of our training procedure are relegated to Appendix \ref{apx:technical}.

In Figure \ref{fig:miscalibrated-data} we apply class weights $\beta_1=0.5,0.9,0.99$ and plot decile-binned calibration curves (\cite{degroot1983comparison,niculescu2005predicting}). The horizontal axis represents the (logged) confidence score, and the vertical axis the corresponding (logged) pneumonia rate in the data. The axes are log-scaled to separate spacing between points and to ease visualization of the relationships, without altering the relationships themselves. Confidence scores are calibrated if the calibration curve coincides with the 45-degree line (the solid diagonal line). Clearly, this holds only when $\beta_1=0.5$, and as the asymmetry in class weights increases, the confidence scores become increasingly miscalibrated. 

\begin{figure}
    \centering
\includegraphics[scale=0.3]{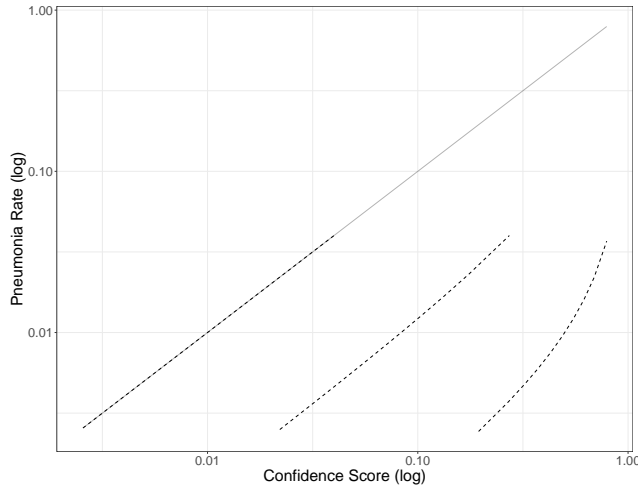}
    \caption{
    \label{fig:miscalibrated-theory}
    Theoretical calibration curves for loss-calibrated algorithm with varying class weights.
    }
\end{figure}

To understand this systematic failure of calibration it is helpful to think about the implications of successfully minimizing losses. If an algorithm successfully minimizes losses, it must not be possible to lower losses by taking all of the instances in which the confidence score $a$ was provided and using alternative confidence score $a'$ instead. It is this property that defines an algorithm as loss-calibrated.

To formalize this property, we denote $L(a,y)$ as the losses for score $a$ when the label is $y$ and $P^L(a,y)$ as the joint probability of scores and labels when the loss function is $L$.\footnote{For simplicity, we assume throughout that this probability distribution is discrete. In practice, distributions are discrete because datasets are finite. Further, scores are typically binned.}
\begin{definition}[Loss-Calibration]
Confidence scores are \textbf{loss-calibrated} to loss function $L$ if a wholesale switching of scores does not reduce losses according to $L$:
\label{def:loss-calibration}
\begin{equation}
\label{eq:loss-calibration}
    \sum_{y \in Y} P^L (a,y) L(a,y) \geq \sum_{y \in Y} P^L (a,y) L(a',y) \text{ for all } a'\in A
\end{equation}
\end{definition}

One value of introducing this construct is that if an algorithm is loss-calibrated there is a straightforward theoretical prediction for the miscalibration induced by a particular loss function. By way of illustration, Figure \ref{fig:miscalibrated-theory} shows this theoretical relationship between scores and pneumonia rates for relative positive class weights $\beta_1=0.5,0.9,0.99$. The precise form of this function is derived in a more general setting in Proposition \ref{thm:ap-bin} in Section \ref{sec:iml}.

Figure \ref{fig:miscalibrated-all} superimposes Figure \ref{fig:miscalibrated-data} and Figure \ref{fig:miscalibrated-theory} combining the actual and theoretical relationships between confidence scores and pneumonia rates. This figure strongly suggests that while the algorithm becomes increasingly miscalibrated with higher weights, it is always very close to being loss-calibrated.

\begin{figure}
    \centering
    \includegraphics[scale=0.3]{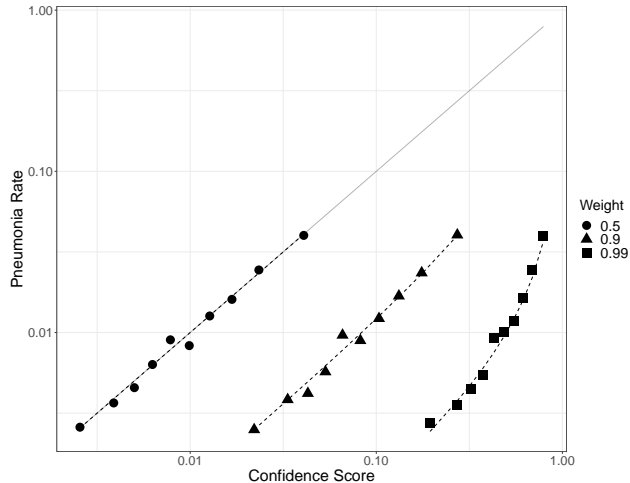}
    \caption{
    \label{fig:miscalibrated-all}
    Relationship between confidence score and pneumonia rate for a loss-calibrated algorithm with varying class weights.}
\end{figure}

Our model-based approach to the incompatibility of calibration and class weighting provides a simple model-based rectification. In typical cases, our proposed solution reduces to a simple analytic correction, the \emph{loss-corrected confidence score}. In this case of weighted binary classification with weight $\beta_1$, the loss-corrected confidence score is:
\begin{equation}
\label{eq:transformation}
g^{\beta_1} (a_1) = \frac{(1-\beta_1)a_1 }{\beta_1 + (1-2\beta_1) a_1}
\end{equation}
Figure \ref{fig:calibrated} plots the loss-corrected confidence scores and indicates that this analytic correction successfully induces calibration. The general mapping is characterized formally in Section \ref{sec:iml}. We also confirm that Figure \ref{fig:miscalibrated-all} and Figure \ref{fig:calibrated} are intimately linked: confidence scores are loss-calibrated if and only if loss-corrected confidence scores are calibrated (Proposition \ref{thm:connect}).

\begin{figure}
    \centering
    \includegraphics[scale=0.3]{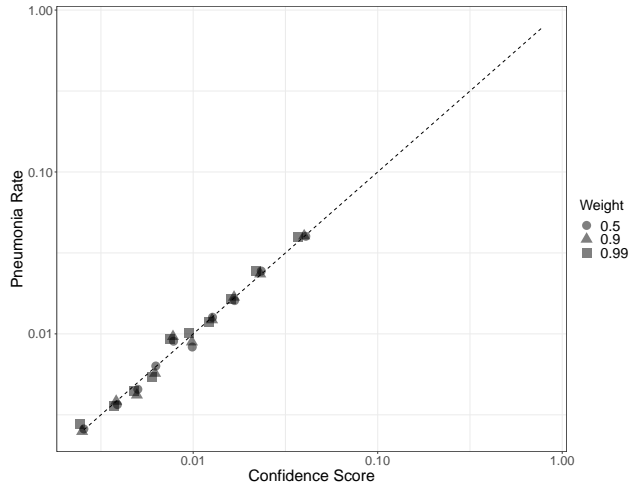}
    \caption{
    \label{fig:calibrated}
    Loss-corrected confidence scores are calibrated.
    Each point plots the average loss-corrected confidence score (log) against the pneumonia rate (log) for a decile of confidence scores}
\end{figure}

\section{Modeling Machine Learning}
\label{sec:iml}

An algorithm converts a set of inputs (training data, training procedure, loss function, etc.) into a scoring rule that is evaluated on a set of labeled observations. We abstract from the complexity of this system by introducing a formal model of an algorithm based on information-theoretic principles. Our approach treats an algorithm as if it follows the \textcite{blackwell1953equivalent} model of experimentation, signal processing, and choice. Technically the algorithm is modeled as an optimizing agent that i) gets signals that provide information about labels, (ii) forms posterior beliefs about labels via Bayesian updating, and iii) converts these posteriors into reported scores. Each stage may be modulated by incentives, in particular those inherent in the scoring loss function $L$. 

Technically, prior beliefs are specified by unconditional probabilities $\pi (y)$ over labels $y \in Y$. Together with a finite set of possible signals $S$ and the conditional probabilities of each signal with each label $\pi (s|y)$ this defines a statistical experiment in the sense of \textcite{blackwell1953equivalent}. Upon observing a signal $s$, the model posits that the algorithm forms posterior beliefs $\gamma^s$ about label $y$ using  Bayes' rule:
\begin{equation}
    \label{eq:posterior}
    \gamma_y^s = \pi (y | s) = \frac{\pi (s | y) }{\pi(s)} \pi (y).
\end{equation}
In the final stage, posterior beliefs $\gamma$ are translated into scores in accordance with a (possibly set-valued) loss-minimizing scoring function:
\begin{equation}
\label{eq:choice}
    c^L (\gamma) = \argmin_{a' \in A} \, \sum_y \gamma_y L(a', y),
\end{equation}
which is well-defined for all posteriors $\gamma$ over the set of labels $Y$. 

\subsection{Loss Calibration}

We begin by providing a tight connection between loss-calibration and consistency with our theoretical model. This formalizes when a machine's behavior can be interpreted as if an algorithm performs a statistical experiment, observes signals, updates in a Bayesian manner, and then optimally scores the resulting posteriors. In turn, separating what an algorithm  learned from how it scored that information forms the basis of our subsequent approach to recovering calibration.

\begin{definition}
\label{def:rep}
For a given loss function $L$, $P^L$ has a \textbf{signal-based representation} (SBR) if there exists a finite signal set $S$, a statistical experiment $\pi$, and a scoring function $\alpha:S\rightarrow A$ such that:  
\begin{enumerate}
    \item Prior beliefs are correct: $\pi (y) = P^L(y)$.
    \item Posterior beliefs satisfy Bayes' rule: $\gamma_y^s = \pi (y | s)$.
    \item Scores are loss-minimizing given posterior beliefs: $\alpha^s \in c^L (\gamma^s)$. \label{item:lm}
    \item Scores are generated by the model: $P^L(a, y) = \sum_{s: \alpha^s = a} \pi (s,y) $.
\end{enumerate}
\end{definition}

\noindent

If $P^L$ has an SBR, then it is \emph{as if} the algorithm optimizes scores given the Bayesian posterior beliefs induced by its statistical experiment. The following result formalizes the equivalence between loss calibration and the existence of such an interpretation.%
\footnote{
Formally, the model and result form a special case of \textcite{caplin2015testable} in which the utility (in our case loss) function is known.}
For this result, we only need apply one minor condition to $P^L$, which is that the distribution of states is not identical for any two scores: $P^L(a,y)\neq P^L(a',y)$ for some $y\in Y$ if $a\neq a'$.
Proofs of all formal results are collected in Appendix \ref{apx:proofs}.

\begin{proposition}
\label{thm:sbr}
    Confidence scores are loss-calibrated to $L$ if and only if $P^L$ has an SBR.
\end{proposition}

\noindent
If confidence scores are loss-calibrated, then a simple SBR is the one corresponding to the observed scores: $S=A$, $\pi(a,y) = P^L(a,y)$, and $\alpha^a = a$. If scores are not loss-calibrated, it is always possible to strictly reduce losses through a wholesale relabeling procedure of some realized score violating \eqref{eq:loss-calibration}, so that no signal-based representation exists  because condition \ref{item:lm} of Definition \ref{def:rep} is violated. 

This representation assumes that machines behave as Bayesian decision makers who report scores in a way that minimizes losses given their beliefs. In the case of proper loss functions --- namely, those that incentivize truthful
revelation of beliefs --- this requires that the machine's confidence scores be calibrated. Early and simple neural networks were shown to have good calibration properties (\cite{niculescu2005predicting}), and there are at least two important classes of modern deep learning algorithms where this also appears to be the case.

First, \textcite{minderer2021calibration} conduct a comprehensive comparison of 180 image classification models and find that the most accurate current models, such as non-convolutional MLP-Mixers (\cite{tolstikhin2021mlp}) and Vision Transformers (\cite{dosovitskiy2021vt}), are not only well-calibrated compared to earlier models, but also that their calibration is more robust to distributions that differ from training.
This body of evidence leads the authors to conjecture that future improvements in model accuracy will only benefit calibration. 
Thus, our SBR representation is likely to remain a useful foundation for modeling machines even as (and because) they become increasingly complex. 

Second, it has been shown that calibration of deep neural nets is improved with regularization procedures such as weight decay (\cite{guo2017calibration}) as well as model ensembling (\cite{lakshmi2017ensemble}). This is consistent with the intuition that both of these procedures reduce overconfidence. As a result, our SBR representation may also be applicable for deep convolutional neural networks that employ standard regularization procedures. 

\subsection{Incentivizing Miscalibration}

To characterize incentive effects on confidence scores we study the optimal posterior scoring choice rule \eqref{eq:choice} as we add class weights to a baseline binary loss function $L$ that is differentiable and \textit{strictly proper}. Such a loss function induces a uniquely optimal action, which is to truthfully score any posterior. By \emph{truthful} we mean that $c^L (\gamma) = \gamma$. 
We begin with the setting of binary classification and adopt the characterization introduced in machine learning by \textcite{buja2005loss} and with origins in psychometrics (\cite{shuford1966admissible}). Being strictly proper corresponds to the loss function being incentive-compatible for belief elicitation (see \cite{schotter2014belief} for a review).

\begin{definition}
\label{def:proper}
    A loss function $L (a_1,y)$ defined on $(0,1) \times \{0,1\}$
    is \textbf{differentiable and strictly proper} if it satisfies 
    differentiability in $a_1$ for each $y \in \{0,1\}$ with derivatives satisfying: 
    \begin{equation}
        \label{eq:proper}        
        \frac{\partial L(a_1,1)}{\partial a_1} = w(a_1) (a_1-1), \quad
        \frac{\partial L(a_1,0)}{\partial a_1} = w(a_1) a_1
    \end{equation}
    for some positive weight function $w(a_1) > 0$ on $(0,1)$, and with $\int_{\varepsilon}^{1-\varepsilon} w (a_1) da_1 < \infty$ for all $\epsilon > 0$.
\end{definition}
\noindent
This class includes standard loss functions such as squared error
and cross-entropy. 
Proper weighting functions are of interest because they incentivize unbiased scoring for symmetrically weighted outcomes.
Thus, restricting to proper loss functions, calibration (Definition \ref{def:calibration})
and loss calibration (Definition \ref{def:loss-calibration}) are equivalent.  

An algorithm that has a signal-based representation would generate calibrated scores if its posterior beliefs were correct on average and truthfully scored $c^L (\gamma) = \gamma$. 
However, the situation for a general loss function is quite different, given that the incentives in mapping posteriors to confidence scores vary with the loss function. 

For loss functions that are not strictly proper, the machine may be incentivized to predict something other than its posterior beliefs.
Consider weighted loss functions that are defined by beginning with a differentiable and strictly proper binary loss function $L$ and reweighting the positive class $\beta_1 \in [0,1]$:
\begin{equation}
    \label{eq:b-loss-bin}
    L^{\beta_1} (a_1,y) = \beta_1 y L(a_1, y) + (1-\beta_1) (1-y) L(a_1,y).
\end{equation}
The following proposition (established in Appendix \ref{apx:proofs}) provides the optimal scoring rule for this weighted loss function.
For simplicity of notation and since the function $L$ is fixed throughout, we suppress it from the superscript on the choice function, so that $c^{\beta_1} (\gamma_1)\equiv c^{L^{\beta_1}} (\gamma_1)$.

\begin{proposition}[Optimal Posterior Scores, Binary Class]
\label{thm:ap-bin}
    Suppose $L$ is a differentiable and strictly proper binary loss function. 
    For all $\gamma_1, \beta_1 \in (0,1)$, the optimal scoring rule for weighted loss $L^{\beta_1}$ defined by \eqref{eq:b-loss-bin} is:
    \begin{equation}
        \label{eq:choice-b1}
        c^{\beta_1} (\gamma_1)=\frac{\beta_1 \gamma_1}{1-\beta_1-\gamma_1+2\beta_1 \gamma_1} .
    \end{equation}
\end{proposition}

\noindent
The optimal scoring rule \eqref{eq:choice-b1} is single-valued, symmetric in its arguments $(\beta_1,\gamma_1)$, and strictly increasing in $\gamma_1$ for every $\beta_1 \in (0,1)$.
It clarifies how the incentives provided to the algorithm are modulated by $\beta_1$. 
When $\beta_1>0.5$, there is an incentive to overscore $c^{\beta_1} (\gamma_1) > \gamma_1$,  
whereas when $\beta_1<0.5$ there is an incentive to underscore $c^{\beta_1} (\gamma_1) > \gamma_1$ all interior posteriors $\gamma_1 \in (0,1)$. 
The impact of $\beta_1$ on optimal scores can be quite strong, 
especially since class weights are frequently used in settings where class imbalance is large.%
For example, at posterior belief $\gamma_1=0.5$, the optimal prediction for a given $\beta_1$ is $ c^{\beta_1} (0.5) = \beta_1$. In the case of our pneumonia application where less than 2\% of labels are positive, inverse probability weighting would yield $\beta_1 > 0.98$.

Proposition \ref{thm:ap-bin} suggests a novel graph that we call the \emph{loss calibration curve}, which overlays the calibration curve with the theoretical map between posterior probabilities $\gamma^1$ and optimal posterior scores, in this case $c^{\beta_1} (\gamma_1)$. If an algorithm is loss-calibrated, the calibration curve should match the theoretical map between probabilities and optimal predictions. This is illustrated by Figure \ref{fig:miscalibrated-all}.

We now generalize the intuition to the case of $n$ labels by using the fact that many richer classification problems can be expressed as an aggregate of simpler binary problems. 
Consider weighted loss functions that are defined by beginning with a differentiable and strictly proper binary loss function $L$ and weighting
according to a matrix $\beta \in \mathbb{R}^{n \times n}$: 
\begin{equation}
\label{eq:b-loss-multi}
    L^\beta (a,y) = \sum_{y' \in Y} \beta_{y,y'} L(a_{y'}, I \{ y = y' \} ).
\end{equation}
In this case we have the following corollary, which follows from observing that the necessary and sufficient condition for an optimal score in each dimension $y$ is expressible as in \eqref{eq:b-loss-bin} for an average weight and a class indicator function. 

\begin{proposition}[Optimal Posterior Scores, Multi-Class]
\label{thm:ap-multi}
    Suppose $L$ is a smooth and strictly proper binary loss function. 
    For all posteriors $\gamma$ and positive weight matrices $\beta$, the optimal scoring rule for $L^{\beta}$ defined by \eqref{eq:b-loss-multi} is:
\begin{equation}
    c_y^\beta (\gamma) 
    =
    \frac{\gamma_y \beta_{y,y}}{\sum_{y' \in Y} \gamma_{y'} \beta_{y', y}}.
\end{equation}
\end{proposition}
\noindent
The intuition of incentivizing miscalibration is the same as previously in Proposition \ref{thm:ap-bin}, except that it now depends on relative weights across multiple classes.

\subsection{Recovering Calibration}\label{sec:calibration}

To recover calibration, we propose inverting the optimal posterior scoring function. 

\begin{proposition}[Loss-Corrected Confidence Score for an Invertible Scoring Rule]
    \label{thm:invert}
    Suppose $c^L (\gamma)$ is single-valued and invertible and confidence scores are loss-calibrated.
    For any $a$ such that $P(a) > 0$, the posterior distribution over labels is recovered by inverting the choice rule:
    \begin{equation}
        \label{eq:invert}
        P (\cdot | a) = (c^L)^{-1} (a).
    \end{equation}

\end{proposition}

\noindent
We call the inverted score $(c^L)^{-1} (a)$ the loss-corrected confidence score.
The following simple result clarifies the relationship between calibration, loss calibration, and loss-corrected confidence scores.
\begin{proposition}
\label{thm:connect}
    Suppose $c^L (\gamma)$ is single-valued and invertible. Then confidence scores are loss-calibrated if and only if loss-corrected confidence scores are calibrated.
\end{proposition}

\noindent
An advantage of our loss-correction calibrating method is that it is analytical and based on the loss function set by the researcher. 
For example, in the case of binary weighted loss \eqref{eq:b-loss-bin}, this yields the correction introduced in \eqref{eq:transformation}.

\begin{corollary}[Loss-Corrected Confidence Score for Binary Weighted Loss]
\label{thm:correct-bin}
    Suppose $L$ is a smooth and strictly proper binary loss function. 
    For all $a_1, \beta_1 \in [0,1]$, the loss-corrected confidence score for $L^{\beta_1}$ defined by \eqref{eq:b-loss-bin} is:
    \begin{equation}\label{eq:loss-c-bin}
    (c^{\beta_1})^{-1} (a_1) = \frac{(1-\beta_1)a_1}{\beta_1+(1-2 \beta_1) a_1 }.
    \end{equation}
\end{corollary} 

\noindent
The loss-corrected confidence scores for our application are plotted in Figure \ref{fig:calibrated}. 
The visual evidence is suggestive of calibration, which is consistent with the loss calibration suggested by Figure \ref{fig:miscalibrated-all} and the connections between calibration, loss calibration, and loss-corrected confidence scores highlighted in Proposition \ref{thm:connect}. 

\section{Relation to the Literature}
\label{sec:conclusion}

There are two alternative ways to arrive at mathematically-equivalent expressions for the loss-corrected confidence scores for binary weighted loss given by (\ref{eq:loss-c-bin}). 
These alternative methods have different interpretations and applications, but they are technically related because they are also based on Bayes rule and optimization, respectively.

The first method is correcting for the distortions of resampling, as in \textcite{pozzolo2015calibration}. Adapted to our notation, their correction formula is:
\begin{equation}\label{eqn:correct-prior}
h^{\delta} (a_1) = \frac{\delta a_1}{1 + (\delta-1) a_1}
\end{equation}
where $\delta$ is the ratio of positive to negative instances in the dataset.
A simple transformation shows that there is a one-to-one relationship between (\ref{eq:loss-c-bin}) and (\ref{eqn:correct-prior}), which is given by $\delta=\frac{1-\beta}{\beta}$.
For example, resampling when there is 1 positive instance for every 10 negative instances is equivalent to having a weight of $\beta=10/11$.  
This one-to-one relationship shows that, if an algorithm has an SBR, there is an amount of resampling that distorts confidence scores in the same way as reweighting distorts confidence scores.  
This is also consistent with the equivalence in \textcite{breiman1984class} between the effects of varying prior probabilities and costs of mistakes across classes.

While there is one-to-one relationship between (\ref{eq:loss-c-bin}) and (\ref{eqn:correct-prior}), these corrections solve fundamentally different problems. The correction in \textcite{pozzolo2015calibration} adjusts posteriors to account for incorrect priors, while our correction adjusts inflated confidence scores to account for weights in the loss functions. 
Since resampling and reweighting also affect model training, in practice it could be that one procedure leads to a calibrated model, whereas another does not. 


Furthermore, the corrections are based on fundamentally different logic. 
Their correction is grounded in Bayes rule and the prior-probability adjustment formulas of \textcite{saerens2002prior} and \textcite{elkan2001cost} and on the logic of sample selection. 
On the other hand, our correction is grounded in the logic of optimization as a function of subjective beliefs.
This alternative perspective provides a simple and generalizable approach for recovery of probabilities that transcends the applications of resampling or reweighting, the use of otherwise proper scoring rules, or the binary class setting. 

The second way to arrive at a mathematically equivalent expression for (\ref{eq:loss-c-bin}) is through an objective probability framework that replaces our subjective posteriors $\gamma$ with the objective probabilities $P(y|x)$ and considers trained models of infinite expressive power, as in for example \cite{buja2005loss}.\footnote{We thank an anonymous reviewer for raising this point.} 
With infinite expressive power, an algorithm should be loss calibrated because, as in an SBR, it should respond optimally given the loss function.
However, framing the calibration problem in the subjective terms of an SBR better accommodates environments with finite data or limited expressive power.
For example, if there is a deterministic label for each instance, a trained model with infinite expressive power would be trivially loss calibrated, as it would only report confidence scores of 0 or 1. Instead, our concept of SBR captures the fact that the algorithm can be uncertain even if the ground truth is certain, in part because the dataset is not infinitely large. 
In such cases, loss calibration is not immediate and requires a theoretical model like SBR.

Our subjective model can also be used to address interesting complications that arise from finite data.
For example, consider the question raised by \textcite{provost2000imbalance}: does reweighting or resampling at the training stage systematically affect what the machine learns, rather than just how the machine reports its information?
To help tackle this open question, \textcite{caplin2022explain} expand on our model by assuming the algorithmic decision-maker optimally chooses what to learn (chooses what signal structure to use) based on resource constraints. By positing both \emph{what} the machine learns and \emph{why}, this extended model allows for further analytic tractability and counterfactual prediction. Together, our papers provide an \emph{Inverse Machine Learning} (IML) approach in which the algorithm's deep parameters are recovered.

Summarizing, despite the fact that these two alternative methods arrive at mathematically-equivalent expressions for the loss-corrected confidence scores for binary weighted loss, our approach rooted in optimization and subjective beliefs offers a different perspective. This change in perspective provides three main advantages. First, it provides a set of tools for understanding the simple correction for binary classification with weighted loss. 
Second, a subjective framework is natural for theoretically modeling algorithms constrained by finite data or limited expressive power. Third, it forms the foundation for a wide class of subjective models of machine learning.

\printbibliography

@book{breiman1984class,
  title={Classification and Regression Trees},
  author={Breiman, Leo and Friedman, Jerome H and Olshen, Richard A and Stone, Charles J},
  year={1984},
  publisher={Wadsworth, Belmont, CA, USA}
}

@article{buja2005loss,
  title={Loss functions for binary class probability estimation and classification: Structure and applications},
  author={Buja, Andreas and Stuetzle, Werner and Shen, Yi},
  journal={Working draft, November},
  volume={3},
  pages={13},
  year={2005},
  publisher={Citeseer}
}

@inproceedings{niculescu2005predicting,
  title={Predicting good probabilities with supervised learning},
  author={Niculescu-Mizil, Alexandru and Caruana, Rich},
  booktitle={Proceedings of the 22nd international conference on machine learning},
  pages={625--632},
  year={2005}
}

@inproceedings{lakshmi2017ensemble,
    title={Simple and scalable predictive uncertainty estimation using deep ensembles},
    author={Lakshminarayanan, Balaji AND Pritzel, Alexender AND Blundell, Charles},
    booktitle={Advances in neural information processing systems 31},
    year={2017}
}

@inproceedings{dosovitskiy2021vt,
    title={An image is worth 16x16 words: Transformers for image recognition at scale},
    author={Dosovitskiy, Alexey AND Beyer, Lucas AND Kolesnikov, Alexander AND Weissenborn, Dirk AND Zhai, Xiaohua AND Unterthiner, Thomas AND Dehghani, Mostafa AND Minderer, Matthias AND Heigold, Georg AND Gelly, Sylvain AND Uszkoreit, Jakob AND Houlsby, Neil},
    booktitle={International conference on learning representations},
    year={2021}
    }

@inproceedings{tolstikhin2021mlp,
    title={MLP-mixer: An all-MLP architecture for vision},
    author={Tolstikhin, Ilya O. AND Houlsby, Neil AND Kolesnikov, Alexander AND Beyer, Lucas AND Zhai, Xiaohua AND Unterthiner, Thomas AND Yung, Jessica AND Steiner, Andreas AND Keysers, Daniel AND Uszkoreit, Jakob AND Lucic, Mario AND Dosovitskiy, Alexey},
    booktitle={Advances in neural information processing systems 34}, 
    year={2021}
}

@article{rajpurkar2018chexnext,
%    doi = {10.1371/journal.pmed.1002686},
    author = {Rajpurkar, Pranav AND Irvin, Jeremy AND Ball, Robyn L. AND Zhu, Kaylie AND Yang, Brandon AND Mehta, Hershel AND Duan, Tony AND Ding, Daisy AND Bagul, Aarti AND Langlotz, Curtis P. AND Patel, Bhavik N. AND Yeom, Kristen W. AND Shpanskaya, Katie AND Blankenberg, Francis G. AND Seekins, Jayne AND Amrhein, Timothy J. AND Mong, David A. AND Halabi, Safwan S. AND Zucker, Evan J. AND Ng, Andrew Y. AND Lungren, Matthew P.},
    journal = {PLOS Medicine},
    %publisher = {Public Library of Science},
    title = {Deep learning for chest radiograph diagnosis: A retrospective comparison of the CheXNeXt algorithm to practicing radiologists},
    year = {2018},
    month = {11},
    volume = {15},
%    url = {https://doi.org/10.1371/journal.pmed.1002686},
    pages = {1--17}
}

@article{saerens2002prior,
    title = {Adjusting the outputs of a classifier to new a prior probabilities: a simple procedure},
    author = {Saerens, Marco AND Latinne, Patrice AND Decaestecker, Christine}, 
    journal = {Neural computation},
    year = {2002}, 
    volume = {14}, 
    number = {1},
    pages = {21--41}
}

@article{platt1999calibration,
    title={Probabilistic outputs for support vector machines and comparisons to regularized likelihood methods},
    author={Platt, John}, 
    year={1999},
    journal={Advances in large margin classifiers},
    pages={61--74},
    volume={10},
    number={3}}

@techreport{caplin2022explain,
  title={Modeling Machine Learning},
  author={Caplin, Andrew and Martin, Daniel and Marx, Philip},
  year={2022},
  institution={working paper}
}

@article{jiang2012medcal,
    title={Calibrating predictive model estimates to support personalized medicine},
    author={Jiang, Xiaoqian and Osl, Melanie and Kim, Jihoon and Ohno-Machado, Lucila},
    journal={Journal of the American Medical Informatics Association},
    pages={263--274},
    year={2012},
    volume={19},
    number={2}
}

@article{cosmides1996humanstat,
    title={Are humans good intuitive statisticians after all? rethinking some conclusions from the literature on judgment under uncertainty},
    author={Cosmides, Leda and Tooby, John},
    journal={Cognition},
    pages={1--73},
    year={1996},
    volume={58},
    number={1}
}

@article{kompa2021medcal,
    title={Second opinion needed: communicating uncertainty in medical machine learning},
    author={Kompa, Benjamin and Jasper Snoek and Beam, Andrew L},
    journal={NPJ Digital Medicine},
    pages={1--6},
    year={2021},
    volume={4},
    number={1}
}

@inproceedings{pozzolo2015calibration,
    title={Calibrating Probability with Undersampling for Unbalanced Classification},
    author={Pozzolo, Andrea Dal and Caelen, Olivier and Johnson, Reid A and Bontempi, Gianluca},
    booktitle={IEEE Symposium Series on Computational Intelligence},
    year={2015},
    pages={159-166}
}

@inproceedings{elkan2001cost,
    title={The foundations of cost-sensitive learning},
    author={Elkan, Charles},
    booktitle={International Joint Conference on Artificial Intelligence (IJCAI)},
    year={2001},
    pages={973--978}
}

@inproceedings{provost2000imbalance,
    title={Machine learning from imbalanced datasets 101},
    author={Provost, Foster},
    booktitle={Proceedings of the AAAI 2000 Workshop on Imbalanced Datasets},
    volume={68},
    year={2000},
    pages={1--3}
}

@inproceedings{deng2009imagenet,
    title={Imagenet: a large-scale hierarchial image database},
    author={Deng, Jia and Dong, Wei and Socher, Richard and Li, Li-Jia and Li, Kai and Fei-Fei, Li},
    booktitle={Computer Vision and Pattern Recognition (CVPR)}, 
    year={2009},
    pages={248--255}
}

@article{bai2021overconfidence,
    title={Don't just blame over-parameterization for over-confidence: theoretical analysis of calibration in binary classification},
    author={Bai, Yu and Mei, Song and Wang, Huan and Xiong, Caiming},
    journal={arXiv preprint arXiv:2102.07856},
    year={2021}
}

@article{liu2022deepprob,
    title={Deep probability estimation},
    author={Liu, Sheng and Kaku, Aakash and Zhu, Weicheng and Leibovich, Matan and Mohan, Sreyas and Yu, Boyang and Huang, Haoxiang and Zanna, Laure and Razavian, Narges and Niles-Weed, Jonathan and Fernandez-Granda, Carlos},
    journal={arXiv preprint arXiv:2111.10734},
    year={2022}
}

@article{huang2016densenet,
    title={Densely connected convolutional networks},
    author={Huang, Gao and Liu, Zhuang and Weinberger, Kilian Q. and van der Maaten, Laurens},
    journal={arXiv preprint arXiv:1608.06993},
    year={2016}
}

@article{kingmaba2014adam,
    title={Adam: a method for stochastic optimization},
    author={Kingma, Diederik and Ba, Jimmy},
    journal={arXiv preprint arXiv:1412.6980}, 
    year={2014}
}

@inproceedings{ioffe2015batch,
    title={Batch normalization: accelerating deep network training by reducing internal covariate shift},
    author={Ioffe, Sergey and Szegedy, Christian},
    booktitle={International Conference on Machine Learning (ICML)},
    year={2015},
    pages={448--456}
}

@inproceedings{minderer2021calibration,
    title={Revisiting the calibration of modern neural networks},
    author={Minderer, Matthias and Djolonga, Josip and Romijnders, Rob and Hubis, Frances and Zhai, Xiaohua and Houlsby, Neil and Tran, Dustin and Lucic, Mario},
    booktitle={Neural Information Processing Systems (NeurIPS)},
    year={2021}
}

@inproceedings{zadroznyelkan2001,
    title={Obtaining calibrated probability estimates from decision trees and naive bayesian classifiers},
    author={Zadrozny, Bianca and Elkan, Charles},
    booktitle={International Conference on Machine Learning (ICML)},
    pages={609--616},
    year={2001}
}

@inproceedings{zadroznyelkan2002,
    title={Transforming classifier scores into accurate multiclass probability estimates},
    author={Zadrozny, Bianca and Elkan, Charles},
    booktitle={Knowledge Discovery and Data Mining (KDD)},
    pages={694--699},
    year={2002}
}

@inproceedings{thai2010cost,
  title={Cost-sensitive learning methods for imbalanced data},
  author={Thai-Nghe, Nguyen and Gantner, Zeno and Schmidt-Thieme, Lars},
  booktitle={The 2010 International joint conference on neural networks (IJCNN)},
  pages={1--8},
  year={2010},
  organization={IEEE}
}

@inproceedings{zadrozny2003cost,
  title={Cost-sensitive learning by cost-proportionate example weighting},
  author={Zadrozny, Bianca and Langford, John and Abe, Naoki},
  booktitle={Third IEEE international conference on data mining},
  pages={435--442},
  year={2003},
  organization={IEEE}
}

@article{shuford1966admissible,
  title={Admissible probability measurement procedures},
  author={Shuford, Emir H and Albert, Arthur and Edward Massengill, H},
  journal={Psychometrika},
  volume={31},
  number={2},
  pages={125--145},
  year={1966},
  publisher={Springer}
}

@article{degroot1983comparison,
  title={The comparison and evaluation of forecasters},
  author={DeGroot, Morris H and Fienberg, Stephen E},
  journal={Journal of the Royal Statistical Society: Series D (The Statistician)},
  volume={32},
  number={1-2},
  pages={12--22},
  year={1983},
  publisher={Wiley Online Library}
}

@article{rajpurkar2017chexnet,
    title={CheXNet: radiologist-level pneumonia detection on chest x-rays with deep learning},
    author={Rajpurkar, Pranav and Irvin, Jeremy and Zhu, Kaylie and Yang, Brandon and Mehta, Hershel and Duan, Tony and Ding, Daisy and Bagul, Aarti and Langlotz, Curtis and Shpanskaya, Katie and Lungren, Matthew P and Ng, Andrew Y},
    journal={arXiv preprint arXiv:1711.05225},
    year={2017}
}

@inproceedings{wang2017chestxray,
    title-={Chestx-ray8: hospital-scale chest x-ray database and benchmarks on weakly-supervised classification and localization of common thorax diseases},
    author={Wang, Xiaosong and Peng, Yifan and Lu, Le and Lu, Zhiyong and Bagheri, Mohammadhadi and Summers, Ronald M},
    booktitle={Proceedings of the IEEE conference on computer vision and pattern recognition},
    year={2017},
    pages={2097--2106}
}

@article{raghu2019algorithmic,
  title={The algorithmic automation problem: Prediction, triage, and human effort},
  author={Raghu, Maithra and Blumer, Katy and Corrado, Greg and Kleinberg, Jon and Obermeyer, Ziad and Mullainathan, Sendhil},
  journal={arXiv preprint arXiv:1903.12220},
  year={2019}
}

@article{schotter2014belief,
  title={Belief elicitation in the laboratory},
  author={Schotter, Andrew and Trevino, Isabel},
  journal={Annu. Rev. Econ.},
  volume={6},
  number={1},
  pages={103--128},
  year={2014},
  publisher={Annual Reviews}
}

@inproceedings{guo2017calibration,
  title={On calibration of modern neural networks},
  author={Guo, Chuan and Pleiss, Geoff and Sun, Yu and Weinberger, Kilian Q.},
  booktitle={International Conference on Machine Learning},
  pages={1321--1330},
  year={2017},
  organization={PMLR}
}

@article{blackwell1953equivalent,
  title={Equivalent comparisons of experiments},
  author={Blackwell, David},
  journal={Annals of Mathematical Statistics},
  pages={265--272},
  year={1953}
}

@article{caplin2015testable,
  title={A testable theory of imperfect perception},
  author={Caplin, Andrew and Martin, Daniel},
  journal={The Economic Journal},
  volume={125},
  number={582},
  pages={184--202},
  year={2015},
  publisher={Wiley Online Library}
}

\newpage
\appendix
\section{Appendix for ``Calibrating for Class Weights by Modeling Machine Learning''}
\subsection{Proofs}
\label{apx:proofs}

The following observation will be useful for subsequent results. It is immediate upon observing that the optimal solution is unchanged upon dividing the minimand by a constant $P^L(a) > 0$ and invoking the law of conditional probability, $P^L (y|a) = P^L(a,y) / P^L(a)$.

\begin{lemma}
\label{thm:lc-lemma}
$P^L (a,y)$ is loss-calibrated:
\[
a \in \argmin_{a' \in A} \sum_{y \in Y} P^L (a,y) L(a',y)
\quad \text{for all $a$}
\]
if and only if:
\[
a \in \argmin_{a' \in A} \sum_{y \in Y} P^L (y|a) L(a',y) 
\quad \text{for all observed $a: P^L(a) > 0$.}
\]
\end{lemma}

\noindent
In what follows, we recall and prove the formal propositions presented in the text.

\begin{proposition2}
    $P^L$ is loss-calibrated if and only if it has an SBR.
\end{proposition2}

\begin{proof}[Proof of Proposition \ref{thm:sbr}]
(Only if:) Suppose $P^L$ is loss-calibrated. 
As observed in the text, it suffices to show that $S=A$, $\pi(a,y) = P^L(a,y)$, and $\alpha^a = a$ form an SBR representation. 
Conditions 1 and 4 of Definition \ref{def:rep} are immediate. 
Condition 2 simply requires that posterior beliefs $\gamma_y^a$ satisfy Bayes' rule, so that the remaining condition 3 is satisfied if:
\[
    \alpha^a \in c^L ( P^L (y|a)).
\]
By loss calibration and Lemma \ref{thm:lc-lemma}, 
\[
    a \in \argmin_{a' \in \mathbb{R}^n} \sum_y P^L (y | a) L(a',y)
\]
Invoking $\alpha^a = a$ and the definition of the optimal scoring rule $c^L$, this yields the desired conclusion.

(If:) We show the contrapositive. Suppose $P^L$ is not loss-calibrated and fix a score $a$ that is observed $P^L(a) > 0$ and for which loss calibration is violated. 
By Lemma \ref{thm:lc-lemma},
\[
    a \notin \argmin_{a' \in \mathbb{R}^n} \sum_y P^L (y | a) L(a',y)
\]
Then for any statistical experiment $\pi$ satisfying conditions 1,2, and 4, it must be that for all signal realizations $s$ such that $\alpha^s = a$, we have $\alpha^s \notin c^L (\gamma^s)$. 
Thus, there does not exist an SBR representation.
\end{proof}

\begin{proposition2}[Optimal Posterior Scores, Binary Class]
    Suppose $L$ is a differentiable and strictly proper binary loss function. 
    For all $\gamma_1, \beta_1 \in (0,1)$, the optimal scoring rule for weighted loss $L^{\beta_1}$ defined by \eqref{eq:b-loss-bin} is:
    \[
        c^{\beta_1} (\gamma_1)=\frac{\beta_1 \gamma_1}{1-\beta_1-\gamma_1+2\beta_1 \gamma_1} 
    \]
\end{proposition2}

\begin{proof}[Proof of Proposition \ref{thm:ap-bin}]
By definition \eqref{eq:choice}, the choice rule $c^{\beta_1} (\gamma_1)$ satisfies:
\[
c^{\beta_1} (\gamma_1) \in \argmin_{a_1 \in \mathbb{R}} 
\gamma_1 \beta_1 L(a_1, 1) + (1-\gamma_1)(1- \beta_1) L(a_1,0)
\]
Since $L(\cdot,y)$ is differentiable for each $y \in \{0,1\}$, a necessary condition for any interior minimizer $a_1^* \in c^{\beta_1} (\gamma_1)$ is: 
\[
\gamma_1 \beta_1 \frac{\partial L(a_1^*,1)}{\partial a_1} + 
(1-\gamma_1) (1-\beta_1) \frac{\partial L(a_1^*,0)}{\partial a_1} = 0 
\]
Invoking the strictly proper characterization \eqref{eq:proper} and collecting terms, 
\[
w (a_1^*) [(1 - \beta_1 - \gamma_1 + 2 \beta_1 \gamma_1) a_1^* - \beta_1 \gamma_1]
= 0
\]
Since $w(a_1^*) > 0$, this condition is also sufficient and implies:
\[
a_1^* = \frac{\beta_1 \gamma_1}{1 - \beta_1 - \gamma_1 + 2\beta_1 \gamma_1}
\]
which yields the desired choice rule. Finally, note that the desired choice rule satisfies $c^{\beta_1} (\gamma_1) \in (0,1)$ for all $\beta_1, \gamma_1 \in (0,1)$.
\end{proof}

\begin{proposition2}[Optimal Posterior Scores, Multi-Class]
    Suppose $L$ is a smooth and strictly proper binary loss function. 
    For all posteriors $\gamma$ and positive weight matrices $\beta$, the optimal scoring rule for $L^{\beta}$ defined by \eqref{eq:b-loss-multi} is:
\[
    c_y^\beta (\gamma) 
    =
    \frac{\gamma_y \beta_{y,y}}{\sum_{y' \in Y} \gamma_{y'} \beta_{y', y}}
\]
\end{proposition2}

\begin{proof}[Proof of Proposition \ref{thm:ap-multi}]
By definition \eqref{eq:choice}, the choice rule $c^{\beta} (\gamma)$ satisfies:
\[
c^{\beta} (\gamma) \in \argmin_{a \in (0,1)^n} \sum_{y' \in Y} \gamma_y L^\beta (a,y')
\]
or, upon substituting the definition of $L^\beta(a,y')$ from \eqref{eq:b-loss-multi} (with the roles of $y$ and $y'$ reversed),
\[
c^{\beta} (\gamma) \in \argmin_{a \in (0,1)^n} \sum_{y,y' \in Y} \gamma_{y'} \beta_{y',y} L (a_{y}, I\{y'=y\})
\]
A necessary condition for a loss-minimizing confidence score $a_y^*$ for outcome $y$ is:
\[
\gamma_y \beta_{y,y} \frac{\partial L(a_y^*,1)}{\partial a_y} + 
( \sum_{y' \neq y} \gamma_{y'} \beta_{y',y} ) \frac{\partial L(a_y^*,0)}{\partial a_y} = 0
\]
Invoking the strictly proper characterization \eqref{eq:proper} and collecting terms, 
\[
w (a_y^*) [ \gamma_y \beta_{y,y} (a_y^* -1) + (\sum_{y' \neq y} \gamma_{y'} \beta_{y',y} ) a_y^* ] 
= 0
\]
As in the proof of the binary case (Proposition \ref{thm:ap-bin}), since $w(a_1^*) > 0$, this condition is also sufficient and implies:
\[
a_y^* = \frac{\gamma_y \beta_{y,y}}{\sum_{y' \in Y} \gamma_{y'} \beta_{y', y}}
\]
which yields the desired choice rule. 
Since $\beta, \gamma \geq 0$, it is evident that this choice rule indeed satisfies $c_y^\beta (\gamma) \in (0,1)$ for all labels $y$.
\end{proof}

\begin{proposition2}[Loss-Corrected Confidence Score for an Invertible Scoring Rule]
    Suppose $c^L (\gamma)$ is single-valued and invertible and confidence scores are loss-calibrated.
    For any $a$ such that $P(a) > 0$, the posterior distribution over labels is recovered by inverting the choice rule:
    \[
        P (\cdot | a) = (c^L)^{-1} (a)
    \]
\end{proposition2}

\begin{proof}[Proof of Proposition \ref{thm:invert}]
Because $c^L (\gamma)$ is single-valued and confidence scores are loss-calibrated and by Lemma \ref{thm:lc-lemma},
\[
a = \argmin_{a' \in A} \sum_{y \in Y} P^L (y|a) L(a',y)
\]
which by definition \eqref{eq:choice} implies:
\[
a = c^L (P (\cdot |a ))
\]
where $P(\cdot |a)$ is the posterior distribution over labels given $a$.
Finally, inverting the optimal posterior scoring rule $c^L (\cdot)$ yields the desired result:
\[
(c^L)^{-1} (a) = P(\cdot|a).
\]
\end{proof}

\begin{proposition2}
    Suppose $c^L (\gamma)$ is single-valued and invertible. Then scores are loss-calibrated if and only if loss-corrected scores are calibrated.
\end{proposition2}

\begin{proof}[Proof of Proposition \ref{thm:connect}]
The ``if'' direction is a rephrasing of Proposition \ref{thm:invert}. For the reverse direction, the loss corrected scores $(c^L)^{-1} (a)$ are calibrated only if \ref{eq:invert} holds.  
Applying the invertible function $c^L (a)$ to both sides of \eqref{eq:invert} and invoking Lemma \ref{thm:lc-lemma} yields the desired result.
\end{proof}

\subsection{Technical Details} 
\label{apx:technical}

We summarize details of the training procedure that generated the data in Figures \ref{fig:miscalibrated-data}, \ref{fig:miscalibrated-all}, and \ref{fig:calibrated}. 
We essentially replicate the pneumonia detection task of \textcite{rajpurkar2017chexnet}, in which a deep neural network was trained on a the ChestX-ray14 dataset of \textcite{wang2017chestxray}.
Our code for model training is adapted from the publicly available codebase of \textcite{rajpurkar2018chexnext}. 
The ChextX-ray14 dataset consists of 112,120 frontal chest X-rays which were synthetically labeled with up to fourteen thoracic diseases. 
In the binary classification task, the labels of interest are pneumonia ($y=1$) or not ($y=0$).
We consider multiple positive class weights $\beta_1 = 0.5,0.9,0.99$, with $0.99$ approximately equal to the inverse probability class weights adopted in \textcite{rajpurkar2017chexnet}. 

As in \textcite{rajpurkar2017chexnet}, we downscale the images to 224 by 224 pixels, adopt random horizontal flipping, and normalize based on the mean and standard deviation of images in the ImageNet dataset (\cite{deng2009imagenet}). 
For each model, we train a 121-layer dense convolutional neural network (DenseNet, \cite{huang2016densenet}) with weights of the network initialized to those pretrained on ImageNet, using Adam with standard parameters 0.9 and 0.999 (\cite{kingmaba2014adam}), using batch normalization (\cite{ioffe2015batch}), and with mini-batches of size 16. 
We use an initial learning rate of 0.0001 that is decayed by a factor of 10 each time the validation loss plateaus after an epoch, and we conduct early stopping based on validation loss. 
Each model was trained using either an Nvidia Tesla V100 16GB GPU or an Nvidia Tesla A100 40GB GPU on the Louisiana State University or Northwestern University high performance computing clusters, respectively. The training of a model typically lasted between one and two hours.

In addition to varying class weights, the main difference in our implementation and the implementation of \textcite{rajpurkar2017chexnet} are our data splits and our recourse to additional ensemble methods to account for randomness in the training procedure.
This use of ensemble methods also likely explains why our transformed confidence scores are calibrated, despite recent evidence that deep neural networks and cross-entropy loss may inherently produce poor calibration because of overconfidence (\cite{bai2021overconfidence}, \cite{liu2022deepprob}).
Specifically, we adopt a nested cross-validation approach where we randomly split the dataset into ten approximately equal folds and then iterate through 70-20-10 train-validation-test splits (the split distribution also used in \cite{wang2017chestxray} and a secondary application of \cite{rajpurkar2017chexnet}). 
We train a total of 400 models, yielding an ensemble of 80 trained models for each observation in the dataset where that observation was in a test fold. 
The final score for each observation in the dataset is then obtained by averaging confidence scores across the observation's ensemble.
This procedure is repeated on the same set of data splits for each weight $\beta=0.5,0.9,0.99$ we consider.

\end{document}